\documentclass[10pt,onecolumn,letterpaper]{article}
\usepackage[utf8]{inputenc}
\usepackage[margin=1.0in]{geometry}

\usepackage{graphicx}
\usepackage{amsmath}
\usepackage{amssymb}
\usepackage{booktabs}
\usepackage{multirow}
\usepackage{authblk}
\usepackage[linesnumbered,ruled,vlined]{algorithm2e}
\SetKwComment{Comment}{/* }{ */}
\usepackage{mathtools}

\DeclarePairedDelimiter\floor{\lfloor}{\rfloor}
\SetKwRepeat{Do}{do}{while}
\usepackage{amsthm}
\newtheorem{theorem}{Theorem}
\newtheorem{lemma}{Lemma}
\newtheorem{corollary}{Corollary}

\newtheorem{definition}{Definition}
\usepackage[pagebackref,breaklinks,colorlinks]{hyperref}

\usepackage[capitalize]{cleveref}
\crefname{section}{Sec.}{Secs.}
\Crefname{section}{Section}{Sections}
\Crefname{table}{Table}{Tables}
\crefname{table}{Tab.}{Tabs.}

\title{GridShift: A Faster Mode-seeking Algorithm for Image Segmentation and Object Tracking}
\author[1]{Abhishek Kumar}
\author[1]{Oladayo S. Ajani}
\author[2]{Swagatam Das}
\author[1]{Rammohan Mallipeddi}
\affil[1]{Department of Artificial Intelligence, Kyungpook National University, Daegu 
37224, South Korea}
\affil[2]{Electronics and Communication Sciences Unit, Indian Statistical Institute Kolkata {700108}, India}

\begin{document}
\maketitle
\begin{abstract}
In machine learning and computer vision, mean shift (MS) qualifies as one of the most popular mode-seeking algorithms used for clustering and image segmentation. It iteratively moves each data point to the weighted mean of its neighborhood data points. The computational cost required to find the neighbors of each data point is quadratic to the number of data points. Consequently, the vanilla MS appears to be very slow for large-scale datasets. To address this issue, we propose a mode-seeking algorithm called GridShift, with significant speedup and principally based on MS. To accelerate, GridShift employs a grid-based approach for neighbor search, which is linear in the number of data points. In addition, GridShift moves the active grid cells (grid cells associated with at least one data point) in place of data points towards the higher density, a step that provides more speedup. The runtime of GridShift is linear in the number of active grid cells and exponential in the number of features. Therefore, it is ideal for large-scale low-dimensional applications such as object tracking and image segmentation. Through extensive experiments, we showcase the superior performance of GridShift compared to other MS-based as well as state-of-the-art algorithms in terms of accuracy and runtime on benchmark datasets for image segmentation. Finally, we provide a new object-tracking algorithm based on GridShift and show promising results for object tracking compared to CamShift and meanshift++.
\end{abstract}

\section{Introduction}
\label{sec:intro}
Mean Shift (MS) is a non-parametric, iterative mode-seeking algorithm for cluster analysis. The well established ability of MS to detect clusters effectively has placed it at the center of several computer vision applications such as unsupervised image and video segmentation~\cite{PARK2009970,Zhou2011,Tao,sylvain,paris}, object tracking~\cite{comaniciu,Leichter2010MeanST,Ning2012RobustMT}, and image processing~\cite{Barash2004ACF,Bigdeli2017DeepMP}. 
\begin{figure*}[!ht]
\centering
\includegraphics[width = 0.75\linewidth]{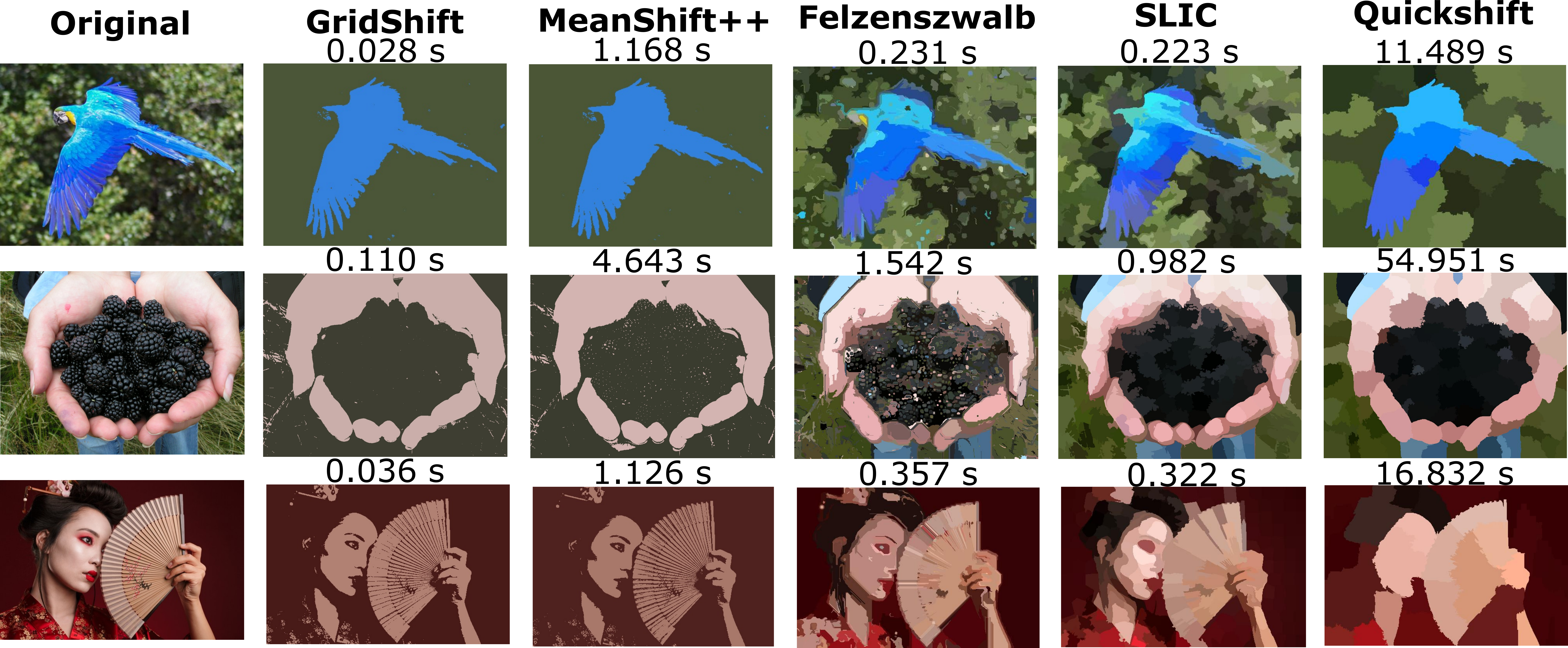}
\caption{\textit{Comparison of five algorithms on three baseline segmentation problems taken from \cite{lin2015microsoft}.} GS returns qualitatively good image segmentation results in all baseline images with lower runtime than other state-of-the-art algorithms: 40x, 10x, 8x, and 500x faster than MS++, Felzenszwalb, SLIC, and Quickshift, respectively.}
\label{fig:1}

\end{figure*}
\begin{figure*}[!ht]
    \centering
    \includegraphics[width = 0.75\linewidth]{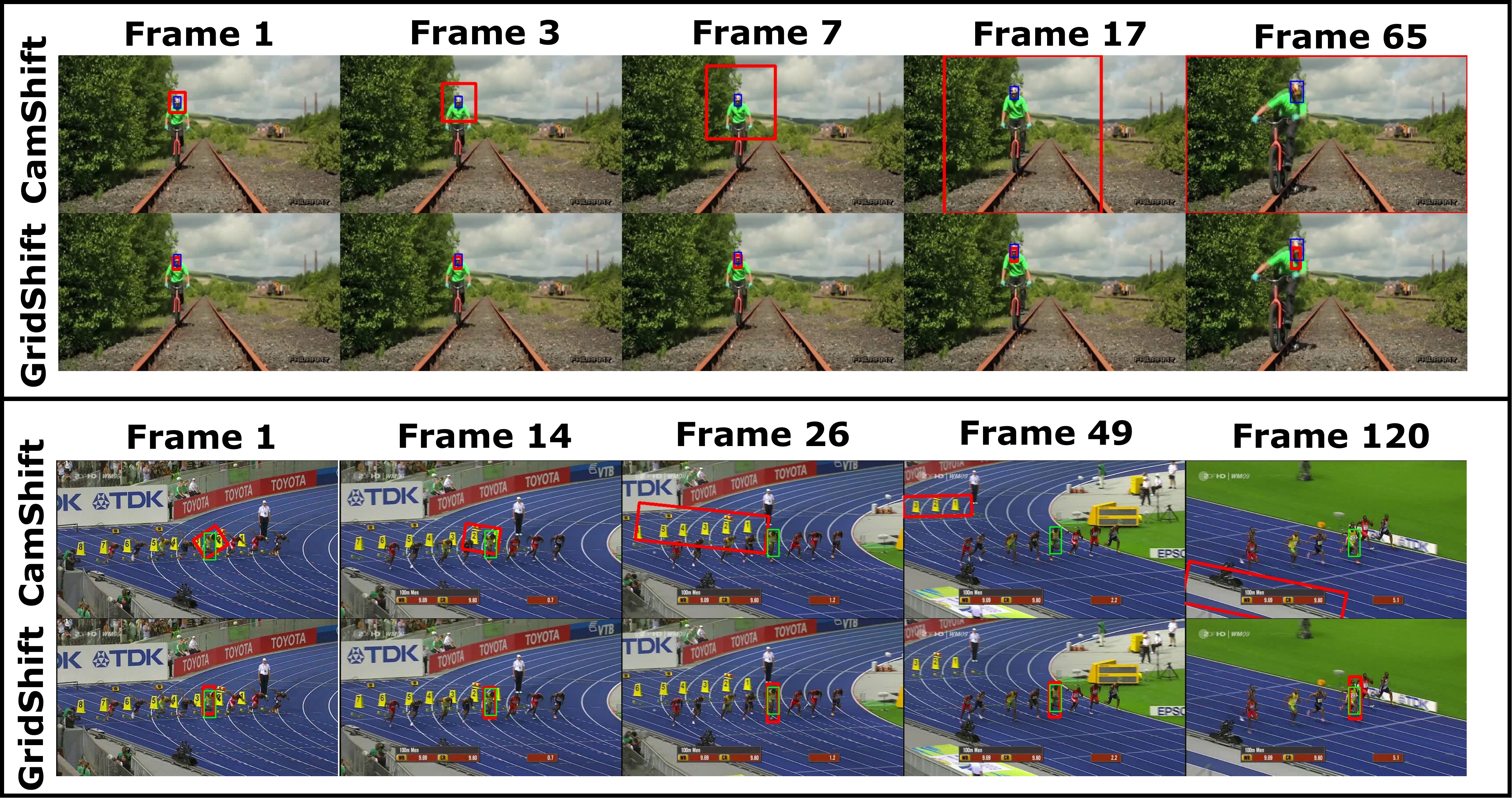}
    \caption{\textit{Comparison of GS and CS on two baseline object tracking problems taken from \cite{WuLimYang13}.} \textbf{Upper row}: we are tracking the cyclist's face (target object shown by blue colored box).  CS algorithm fails to track due to low lighting on the face and background with a high density of green color. On the other hand, GS tracks the face accurately in all the video frames despite these issues. \textbf{Last row}:  we are tracking player Bolt (target object shown by Green color) during the race. Again CS fails to follow Bolt because of his clothes color, which is yellow, and it is distracted by the other objects of yellow color. On the other hand, GS tracks the Bolt accurately without being affected by the other objects of the same color in all video frames.}
    \label{fig:2_new}

\end{figure*}
MS does not need the number of clusters to be specified beforehand. Moreover, it also makes no prior assumption on the probability distribution of the data points. 
Because of such attributes, MS is often preferred over other popular and application-specific algorithms such as $k$-meanss~\cite{ZHAO2018195,ISMKHAN2018402}, Spectral clustering~\cite{Huang}, Felzenszwalb~\cite{Felzenszwalb2004EfficientGI}, and simple linear iterative clustering (SLIC)~\cite{achanta2012slic}. However, the iterative process in MS is computationally expensive since the asymptotic or bounded behavior of each iteration is of the $O(n^{2})$ order. In other words, the computational complexity required for finding neighborhood data points for each point is quadratic to the number of data points. Despite its superior performance in image segmentation, this high computational cost limits the application of MS to high-resolution image segmentation. To address this issue, we put forth an extended variant of MS, named GridShift (GS), with a time complexity that is linear to the number of data points and exponential to the number of features (dimension of data points).

A popular color-based object tracking algorithm, CamShift (CS)~\cite{bradski1998computer,allen2004object}, was derived from the MS algorithm. Although it qualified as a popular unsupervised object tracking procedure, CS has a few major drawbacks which limit its applicability to real-world tracking problems. The main disadvantage is that it only utilized the peak of the back-projected probability distribution. Consequently, it cannot distinguish objects with similar colors and it suffers from tracking failure caused by similar colored objects around the searching window~\cite{jang2021meanshift++}. The other issue is that CS often utilizes only one reference histogram due to computational burden. As a result, changing appearance or lighting conditions affect the tracking performance of CS and it often fails to track the object. In this work, we develop a new object-tracking algorithm based on the GS algorithm by addressing these two issues.

In  MS++~\cite{jang2021meanshift++} and $\alpha$-MS++~\cite{park},  the data space is partitioned into grids, where each data point is associated with an appropriate grid cell. With an increase in iterations, data points shrink towards their nearest mode; therefore, the number of data points associated with an active grid cell is also increased~\cite{cheng1995mean}. The shifted location of all associated data points of a grid cell can be approximated as a single point at the weighted mean of their shifted location (termed centroid of that grid cell)~\cite{jang2021meanshift++, park}.  Thus, this motivates us to develop a new grid-based framework, GridShift,  different than MS++ and $\alpha$-MS++,  where we apply shifting steps of MS directly to approximated centroids of data points. Since multiple data points have the same centroid, we can reduce the computational cost by increasing iterations. Thus, the proposed algorithm is faster than MS++ and $\alpha$-MS++.  Due to the reduced computational time, GS may qualify as an attractive alternative to the existing techniques in computer vision applications because of the increasing resolution of collected data from cameras and sensors and the growing size of modern datasets.

Our major contributions are as follows:
\begin{itemize}
    \item We propose a faster mode-seeking algorithm GS, for unsupervised clustering of large-scale datasets with lower dimensions (which could be very well suited for applications such as image segmentation).
    \item To come out from the closed-loop iteration cycle, GS does not require any stopping criteria.
    \item Empirical analysis indicates that GS provides better or at least comparable clustering results to MS and MS++ with significantly faster computation speed.
    \item An image segmentation experiment on the well-accepted benchmark datasets suggests that GS may indeed outperform some of the popular state-of-the-art algorithms in terms of accuracy and efficiency-- yielding speedup of 40x and 40000x as compared to MS++ and MS respectively.
    \item We also present a new object tracking algorithm based on GS that gradually adapts to the changes in scenes and color distribution. Our experiments suggest that GS-based object tracking algorithm detects objects more accurately than MS++ and CS-based object tracking algorithm.
\end{itemize}
Before delving further into the details of GS, let us provide two warm-up examples to illustrate its potential in image segmentation and object tracking applications.\\
\mbox{~~~}\vspace{-3mm}\\
\textbf{1) A Motivating Example for Image Segmentation:} We utilized three baseline images taken from \cite{WuLimYang13} for unsupervised image segmentation. For comparative analysis, we consider the following algorithms as state-of-the-art: Felzenszwalb~\cite{Felzenszwalb2004EfficientGI}, Quickshift~\cite{Vedaldi2008QuickSA}, SLIC ($k$-means)~\cite{achanta2012slic}, and MS++~\cite{jang2021meanshift++}. We show the outcomes of all these algorithms in Figure \ref{fig:1}. As shown in Figure \ref{fig:1}, GS can produce a similar segmentation result as MS++, but with a $40$x speedup. As compared to other algorithms, GS results in better segmentation with faster computational speed. It is clear from this example that GS provides better segmentation results than other algorithms with lower runtime.\\
\mbox{~~~}\vspace{-3mm}\\
\textbf{2) A Motivating Example for Object Tracking:} For object tracking, we employ CS and GS. In Figure \ref{fig:2_new}, we show the tracking result of these algorithms on two benchmark problems taken from \cite{WuLimYang13}. It is seen from Figure \ref{fig:2_new} that GS is a more robust object tracking algorithm in a complex background environment, especially in large areas with a similar color. On the other hand, CS fails to track in such situations as its performance is highly dependent on the color of the object.
\begin{figure}[!ht]
    \centering
    \includegraphics[width=0.7\linewidth]{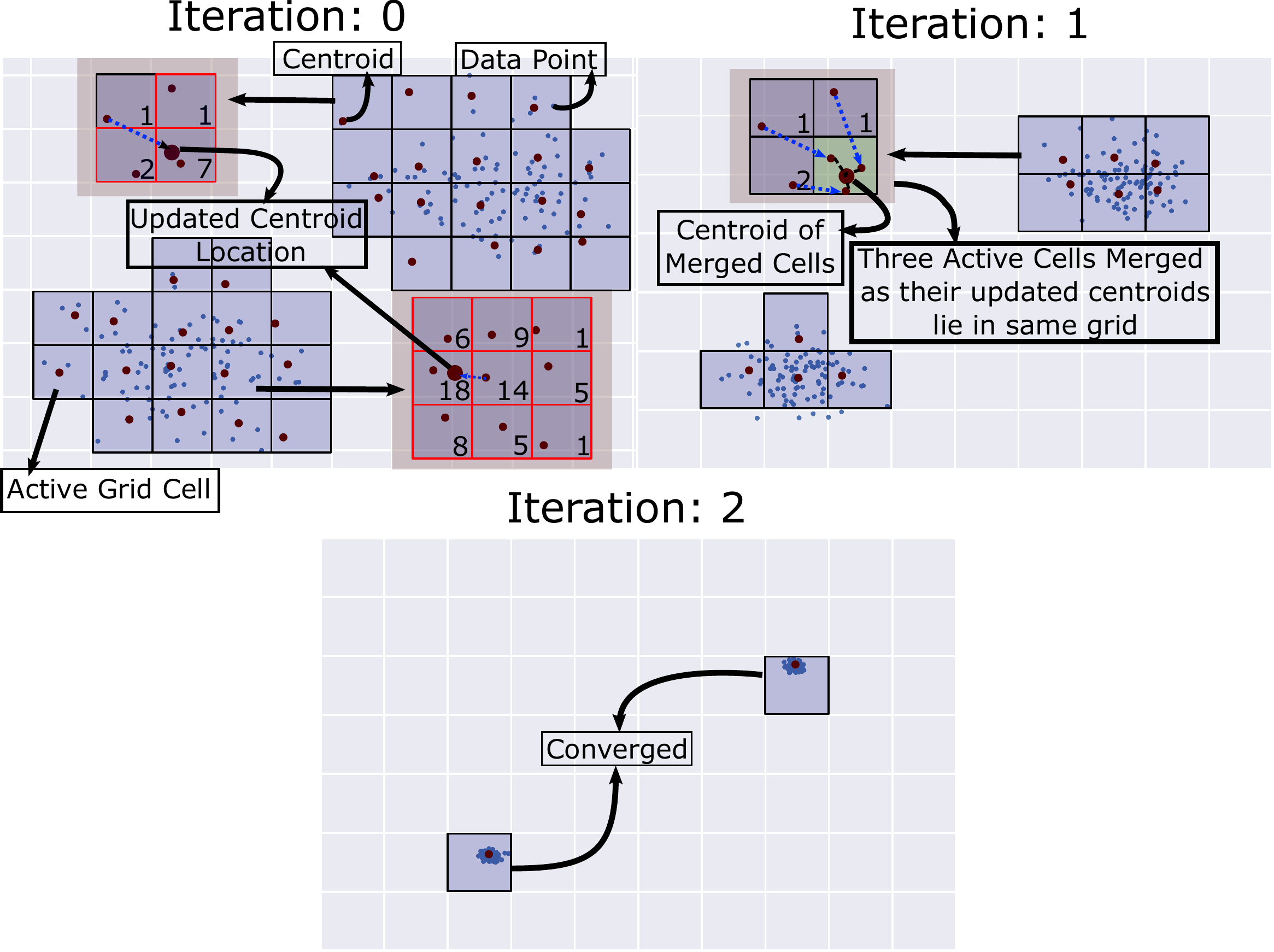}
    \caption{Demonstration of MS and GS parallelly for clustering of a Gaussian distributed 2-D dataset. GS can approximate the MS steps effectively despite reducing the number of centroids in each iteration. Such reduction provides speed up in the runtime compared to the original MS, MS++, and $\alpha$-MS++.}
    \label{fig:GS}
\end{figure}

\section{The GridShift}
In this section, we start by detailing the basic steps of GS and then outline the primary differences between MS++ and GS. 
Finally, we present an object tracking algorithm based on GS.
\subsection{Basic Implementation of the GridShift}
Let $X:= \{x_1,x_2,\ldots,x_n\}$ be a dataset of $n$ data points, where $x_i \in \mathbb{R}^d$. This algorithm assumes that the whole data space is partitioned into equal and disjoint grid cells of predefined length $h$. We categorize these grid cells into two types: active grid cells and inactive grid cells. The grid cells with at least one data point as a resident are treated as active; otherwise, they are considered inactive grid cells. Here, we can label each active grid cell as a cluster where its resident data points are members of that cluster. These grid cells have their own three attributes represented as hash tables: centroid, number of resident data points, and a set of resident data points. All active grid cells update their hash tables in each iteration. Further, GS shifts the location of the active grid cells according to their new centroids. During shifting, some active grid cells may enlist the same location; therefore, they merge into one. These steps are repeated till convergence, i.e., no change in hash tables of all active grid cells. After convergence, GS returns clustering results in terms of active grid cells. A demonstration of the GS steps is depicted in Figure \ref{fig:GS}. We summarize the basic steps of GS in Algorithm \ref{alg:1} and also discuss in more detail as follows.
\begin{itemize}
    \item[\textit{i)}] \textit{Initialization of the active grid cells:} For initialization, we create three empty hash tables for each attribute of active cells: $\mathcal{S}$ (store centroid), $\mathcal{C}$ (store number of resident data points), and $\mathcal{H}$ (store resident data points). We divide each data point by $h$ for creating the initial index of the active grid cells. After that, we calculate the element-wise floor function, which returns a $d$-dimensional integer vector. This $d$-dimensional vector represents the index of the active grid cell. All three hash tables are updated using lines 3-11 of Algorithm \ref{alg:1}.
    \item[\textit{ii)}] \textit{Update centroid of each active grid cells:} To update the centroid of an active grid cell, we utilize its neighboring active grid cells. Here, neighboring grid cells mean the grid cells that touch this active grid cell by a side or a point, i.e., a grid cell with index $j$ is considered as neighboring grid cell of the grid cell with index $i$, if $i = j+v,$ where $v \in \{-1,0,1\}^d$. We update the centroid of an active grid cell using the weighted mean of the centroid of all its neighboring grid cells, i.e.
    \begin{equation} \label{update}
        \mathcal{S}(j) \gets \frac{\sum_{\left(v \in \{-1,0,1\}^d\right)} w_v \mathcal{S}(j+v)}{\sum_{\left(v \in \{-1,0,1\}^d\right)} w_v},
    \end{equation}
    where,
    \begin{equation}
        w_v = \mathcal{C}(j+v).
    \end{equation}
    \item[\textit{iii)}] \textit{Update of the location of the active grid cells:} Each active grid cell needs to update its location or index due to the update of its centroid. The new location is calculated using $\floor*{\mathcal{S}/h}$. 
    \item[\textit{iv)}] \textit{Merger of the some of the active grid cells:} The active grid cells that share the same location (index) need to merge into one grid cell with updated attributes. The centroid of the new grid cell is calculated using the weighted mean of the centroids of all active grid cells that are merged, i.e.
    \begin{equation}
        \mathcal{S}(j) = \frac{\sum_{i \in M} \mathcal{C}(i) \mathcal{S}(i)}{\sum_{i \in M} \mathcal{C}(i)},
    \end{equation}
    where $M$ is the set of indices of grid cells that are merged into one. The other attributes are updated in the following way:
    \begin{equation}
        \begin{split}
            \mathcal{C}(j) = \sum_{i \in M}\mathcal{C}(i),\\
            \mathcal{H}(j) = \bigcup_{i \in M} \mathcal{H}(i).
        \end{split}
    \end{equation}
    \item[\textit{v)}] \textit{Stopping criterion:} The algorithm comes out from the iteration when attributes of the active grid cells will not update, i.e., active grid cells will not have a single neighboring active grid cell to update the grid cell's attributes.
\end{itemize}
 It is worth noting that we optimize the basic steps of the GS during the implementation in Algorithm \ref{alg:1} to reduce extra loops for improving the time complexity. According to the steps used in Algorithm \ref{alg:1}, the time complexity of GS is $\mathcal{O}(m3^d)$ per iteration, where $m$ denotes the number of the active grid cells and it remains non-increasing as the iterations proceed.

\begin{algorithm}
\scriptsize
\caption{GridShift}\label{alg:1}
\textbf{Input:} side length of cells (bandwidth) $h$, $X_n$\;
\textbf{Initialize:} Empty hash tables $\mathcal{S}$ (store centroid), $\mathcal{C}$ (store number of resident data points), and $\mathcal{H}$ (store resident data points)\;
\For {$i \in [1,n]$}{
\eIf{$\floor*{x_i/h}\in \mathcal{S}.keys$}{
$\mathcal{S}\left(\floor*{x_i/h}\right) \gets \frac{\mathcal{C}\left(\floor*{x_i/h}\right)\mathcal{S}\left(\floor*{x_i/h}\right)+x_i}{\mathcal{C}\left(\floor*{x_i/h}\right)+1}$\;
$\mathcal{C}\left(\floor*{x_i/h}\right) \gets \mathcal{C}\left(\floor*{x_i/h}\right)+1$\;
$\mathcal{H}\left(\floor*{x_i/h}\right) \gets  \mathcal{H}\left(\floor*{x_i/h}\right) \cup \{i\}$\;
}{
$\mathcal{S}\left(\floor*{x_i/h}\right) \gets x_i$\;
$\mathcal{C}\left(\floor*{x_i/h}\right) \gets 1$\;
$\mathcal{H}\left(\floor*{x_i/h}\right) \gets  \{i\}$\;
}
}
\Do{$\left(\mathcal{S}.keys+v \right) \cap \mathcal{S}.keys == \emptyset$, where $v\in(\{-1,0,1\}^d-\{0,0,0\})$}{
$\mathcal{S}' \gets \mathcal{S}$, $\mathcal{H}' \gets \mathcal{H}$  and $\mathcal{C}' \gets \mathcal{C}$\;
Empty hast tables $\mathcal{C}$ and $\mathcal{S}$\;
 \For {$j \in \mathcal{S}.keys$}{ 
$\mathcal{S}'(j) \gets \frac{\sum_{\left(v \in \{-1,0,1\}^d\right)} \mathcal{C}'(j+v) \mathcal{S}'(j+v)}{\sum_{\left(v \in \{-1,0,1\}^d\right)} \mathcal{C}'(j+v)}$\;

\eIf{$\floor*{\mathcal{S}'(j)/h}\in \mathcal{S}.keys$}{
$\mathcal{S}\left(\floor*{\mathcal{S}'(j)/h}\right) \gets \frac{\mathcal{C}\left(\floor*{\mathcal{S}'(j)/h}\right)\mathcal{S}\left(\floor*{\mathcal{S}'(j)/h}\right)+\mathcal{S}'(j)}{\mathcal{C}\left(\floor*{\mathcal{S}'(j)/h}\right)+\mathcal{C}'(j)}$\;
$\mathcal{C}\left(\floor*{\mathcal{S}'(j)/h}\right) \gets \mathcal{C}\left(\floor*{\mathcal{S}'(j)/h}\right)+\mathcal{C}'(j)$\;
$\mathcal{H}\left(\floor*{\mathcal{S}'(j)/h}\right) \gets  \mathcal{H}\left(\floor*{\mathcal{S}'(j)/h}\right) \cup \mathcal{H}'(j)$\;
}{

$\mathcal{S}\left(\floor*{\mathcal{S}'(j)/h}\right) \gets \mathcal{S}'(j)$\;
$\mathcal{C}\left(\floor*{\mathcal{S}'(j)/h}\right) \gets \mathcal{C}'(j)$\;
$\mathcal{H}\left(\floor*{\mathcal{S}'(j)/h}\right) \gets  \mathcal{H}'(j)$\;
}
}
}
\textbf{return} $\mathcal{H}$ and $\mathcal{S}$
\end{algorithm}

\subsection{Primary differences between MS++ and GS}
The most recent attempt to speedup MS is MS++ proposed by Jang and Jiang \cite{jang2021meanshift++}. The improvement process in  MS++ involves 1) partitioning the input domain into a grid, 2) assigning each data point to its associated grid, and 3) searching for the estimated mode for each point from its grid as well as grids within its neighborhood. Although MS++ is more than $1000$x faster than MS, Park \cite{park} claimed that without parallel computing, MS++ is not sufficiently fast yet. Therefore, \cite{park} proposed a variant of  MS++  known as $\alpha$-MS++,  which combines the use of an auxiliary hash table,  a speedup factor $(\alpha)$ to reduce the number of iterations required until convergence as well as seeking more accurate modes using more numbers of neighboring grid cells while reducing the size of grid cells to minimize the computational redundancy of MS++.  Although the experimental results in MS++ and $\alpha$-MS++ are considerably faster than MS  and MS++, respectively, we will show that GS is still orders of magnitude faster than both MS++ and $\alpha$-MS++. Although both GS and MS++ use grid-based neighborhood search, the basic framework of GS is different from MS++ in the following aspects:
\begin{itemize}
    \item[\textit{i)}] MS++ updates the location of each data point by using the weighted mean of the data points of $1$-neighboring grid cells. On the other hand, GS updates the centroid of each active data cell by using the weighted mean of the centroid of the $1$-neighboring active grid cells.
    \item[\textit{ii)}] In MS++, the shifted location of all data points associated with the same grid cell has the same value at any particular iteration. In GS, the centroid can also be treated as the location of all data points associated with the same grid cells. However, these centroids' locations may differ from the shifted location calculated in MS++ as GS updates centroids in a sequential manner (shown in Eqn \ref{update}), i.e.,
    \begin{equation}\label{update2}
        \begin{split}
            & \mathcal{S}^{(t)}(j) \gets \frac{\sum_{\left(v \in \{-1,0,1\}^d\right)} w_v \mathcal{S}^{(t)}(j+v)}{\sum_{\left(v \in \{-1,0,1\}^d\right)} w_v},\\
            &\mathcal{S}^{(t+1)}(j) = \mathcal{S}^{(t)}(j), ~ \forall j \in \{1,2,\ldots, k^{(t)}\}
        \end{split}
    \end{equation}
    However, MS++ updates in a parallel fashion which is equivalent to the following equation (in terms of Eqn. (\ref{update2}).
        \begin{equation}\label{update3}
             \mathcal{S}^{(t+1)}(j) \gets \frac{\sum_{\left(v \in \{-1,0,1\}^d\right)} w_v \mathcal{S}^{(t)}(j+v)}{\sum_{\left(v \in \{-1,0,1\}^d\right)} w_v}.
    \end{equation}
    For better understanding, in Fig. \ref{revfig:1}, we show the shift of centroids for both cases on a sample of nine centroids. As shown in this figure, Eqn. \ref{update3} does not take the neighbor's updated location into account in the shifting of centroids. Alternatively, in Eqn. \ref{update2}, the location of neighbors that have already been updated in the shifting sequence is used instead of their older location. In comparison with Eqn. \ref{update3}, Eqn \ref{update2} provides better convergence of data points towards modes. We will add clearer explanations on this in the final version.
    \begin{figure}[!ht]
        \centering
        \includegraphics[width = 0.7\linewidth]{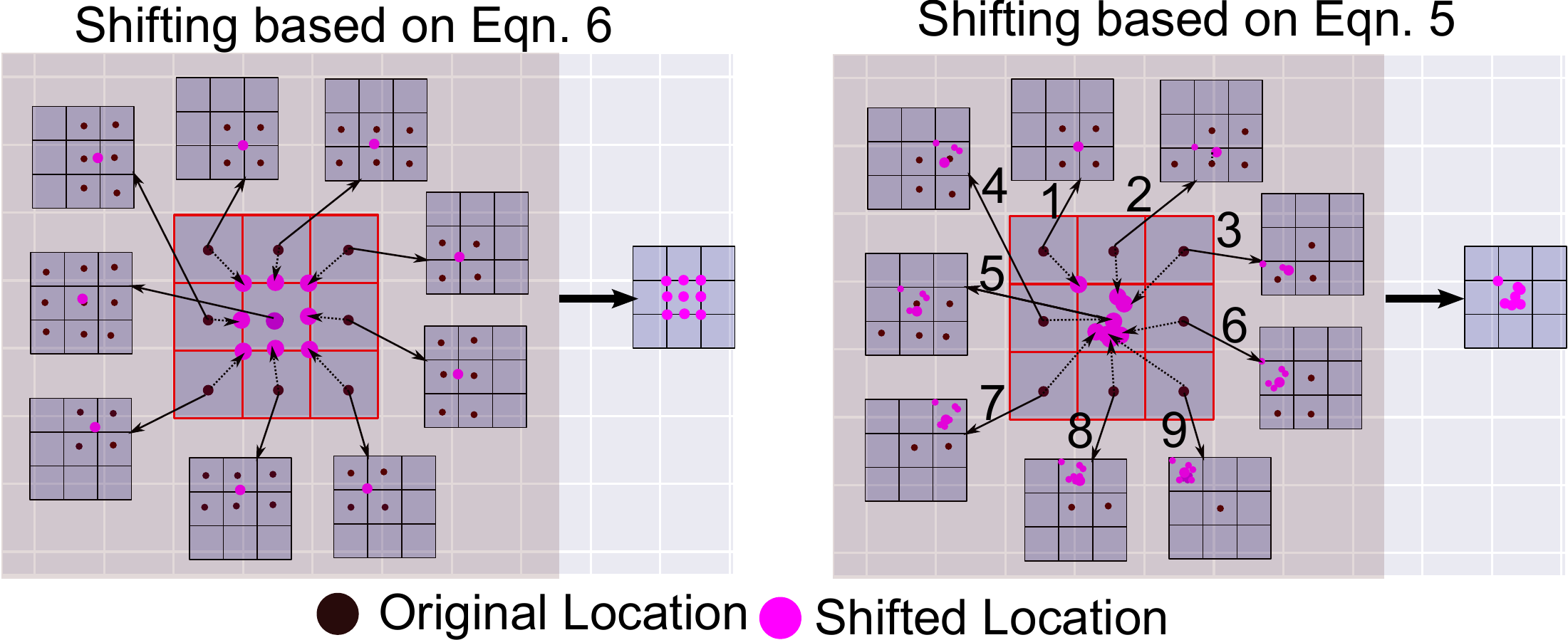}
        \caption{Shifting of centroids using Eqns. 5 and 6. }
        \label{revfig:1}
    \end{figure}
    
    \item[\textit{iii)}] The time complexity of MS++ is $\mathcal{O}(n3^d)$ per iteration, where $n$ is the number of data points. On the other hand, the time complexity of the GS is $\mathcal{O}(m3^d)$ per iteration, where $m << n$ with the increase of iterations. Therefore, GS is faster than MS++ by $n/m_{avg}$ times theoretically, where $m_{avg}$ is the average value of $m$ throughout the iterations.
\end{itemize}
\subsection{GS-based Object Tracking Algorithm}
GS can be a good candidate for object tracking due to its effective mode-seeking behavior and faster runtime. Moreover, GS's clustering results in grid cells are a suitable replacement for the back-projected probability distribution of the histogram~\cite{jang2021meanshift++}. In Algorithm \ref{alg:2}, we show the basic steps of the proposed object tracking algorithm based on GS. The basic steps of this algorithm are as follows.
\begin{algorithm}
\scriptsize
\caption{GS-based object tracking}\label{alg:2}
\textbf{Input:} Sequence of frames $X_0$, $X_1$, $X_2$,$\ldots$, $X_T$; Initial centre $o$; Length $l$; Width $w$; Increment factor $f$; Bandwidth $h$; and Termination tolerance $\eta$\;
\textbf{Run} GS on the color space of frame window $W(o,l,w) \cap X_0$\;
\textbf{Initialize:} $B \gets \{\floor*{c/h}:c\in C\}$, where $C$ is the union of manually selected clusters\;
\For{$i = 1,2,\ldots,T$}{
\Do{ $o$ converges with $\eta$}{
$W' \gets W(o,l/f,w/f)$\;
$R \gets \{x \in W' \cap X_i: \floor*{x/h} \in B\}$\;
\eIf{$R == \phi$}{
$l \gets 1.1l$\;
$w \gets 1.1w$\;
}{
$o \gets$ mean of $(x,y)$-position of $R$'s points in the frame $X_i$\;
$B \gets \{\floor*{c/h}:c\in R\}$\;
\eIf{$\left(\max(l)-\min(l)\right) < l$}{
$l \gets 0.99l$\;
}{
$l \gets 1.01l$\;}
\eIf{$\left(\max(w)-\min(w)\right) < w$}{
$w \gets 0.99w$\;
}{
$w \gets 1.01w$\;}
}
}
\textbf{emit} $W(o,l,w)$ for frame $X_i$\;}
\end{algorithm}
\begin{itemize}
    \item[\textit{i)}] \textit{Initialization:} We apply GS on the color pixel of the initial frame window $W(o,l,w) \cap X_o$, where $W(o,l,w)$ represents a track window with center $o$, length $l$, and width $w$. GS returns different clusters for different objects. We select clusters manually for the targeted object. Generally, we often select clusters with bigger size. After selecting the clusters, we calculate the reference set of grid cells, $R$, where data points of the selected clusters are the resident in the given frame window. Further, we use this reference set of grid cells to track the centers in the new frame.
    \item[\textit{ii)}] \textit{Calculation of new center:} With the change of the frame, the target object also changes the location. According to the change of location of the target object, we need to calculate the new center of the window. The new center for the given frame is calculated iteratively till convergence. For this purpose, we create a track widow W. After that, we find the data points, which are the resident of the grid cells of the R, and create a set B. then, we calculate the center by taking the mean of (x,y)-position of B's data points in the given frame.
    \item[\textit{iii)}] \textit{Update length and width of the tracking window:}  Sometimes the object moves fast enough to drive out the track window in the new frame due to the small size of the track window. In such cases where B is an empty set, we increase the length and width of the search window by 1.1 times to increase the chances of object tracking. On the other hand, the size of the tracking object and its appearance change with frames' changing. To adapt the tracking window accordingly, we increase or decrease the length and/or width (shown in lines 14-21 of Algorithm 2.) 
\end{itemize}
The developed new scheme of object tracking is more robust and effective than the CS algorithm. Its low dependency on the probability distribution of the color histogram of the object and adaptation scheme of the tracking window improves the performance in the object tracking (demonstrated in motivating example 2 in Section 1).
\begin{table}[htbp]
  \centering
   \resizebox{!}{!}{ \begin{tabular}{l|ccc|ccc}\hline
    \multirow{2}[0]{*}{Dataset} & \multicolumn{3}{c}{ARI} & \multicolumn{3}{c}{AMI} \\\cline{2-7}
          & GS    & MS++  & $\alpha$-MS++ & GS    & MS++  & $\alpha$-MS++ \\\hline\hline
    \multirow{2}[0]{*}{Phone Accelerometer} & $\textbf{0.0899}$ & $0.0897$ & $0.0896$ & $0.1923$ & $\textbf{0.1959}$ & $0.1921$ \\
          & $(38.23s)$ & $(1599.34s)$ & $(475.58s)$ & $(45.34s)$ & $(2523.85s)$ & $(646.27s)$ \\\hline
    \multirow{2}[0]{*}{Phone Gyroscope} & $\textbf{0.2401}$ & $0.2354$ & $0.2355$ & $\textbf{0.1837}$ & $0.1835$ & $0.1824$ \\
          & $(110.43s)$ & $(5324.19s)$ & $(934.28s)$ & $(45.39s)$ & $(1723.19s)$ & $(780.35s)$ \\\hline
    \multirow{2}[0]{*}{Watch Accelerometer} & $\textbf{0.1001}$ & $0.0913$ & $0.0908$ & $\textbf{0.2314}$ & $0.2309$ & $0.2301$ \\
          & $(18.38s)$ & $(926.59s)$ & $(319.29s)$ & $(41.86s)$ & $(2500.97s)$ & $(658.37s)$ \\\hline
    \multirow{2}[0]{*}{Watch Gyroscope} & $\textbf{0.1623}$ & $0.1593$ & $0.1602$ & $\textbf{0.1422}$ & $0.1336$ & $0.1397$ \\
          & $(25.23s)$ & $(1247.08s)$ & $(416.85s)$ & $(11.53s)$ & $(484.27s)$ & $(180.24s)$ \\\hline
    \multirow{2}[0]{*}{Still} & $0.7896$ & $0.7899$ & $\textbf{0.7901}$ & $\textbf{0.8602}$ & $0.8551$ & $0.8599$ \\
          & $(0.17s)$ & $(8.23s)$ & $(3.12s)$ & $(0.12s)$ & $(6.23s)$ & $(2.62s)$ \\\hline
    \multirow{2}[0]{*}{Skin} & $\textbf{0.3266}$ & $0.3264$ & $0.3264$ & $\textbf{0.4251}$ & $0.4238$ & $0.4234$ \\
          & $(0.29s)$ & $(12.58s)$ & $(3.28s)$ & $(0.19s)$ & $(9.28s)$ & $(2.92s)$ \\\hline
    \multirow{2}[0]{*}{Wall Robot} & $\textbf{0.1801}$ & $0.1788$ & $0.1748$ & $\textbf{0.3356}$ & $0.3239$ & $0.3336$ \\
          & $(<0.01s)$ & $(0.14s)$ & $(0.08s)$ & $(<0.01s)$ & $(0.42s)$ & $(0.12s)$ \\\hline
    \multirow{2}[0]{*}{Sleep Data} & $\textbf{0.1201}$ & $0.1181$ & $0.1193$ & $\textbf{0.1117}$ & $0.1028$ & $0.1056$ \\
          & $(<0.01s)$ & $(0.01s)$ & $(<0.01s)$ & $(<0.01s)$ & $(0.01s)$ & $(<0.01s)$ \\\hline
    \multirow{2}[0]{*}{Balance Scale} & $0.0799$ & $\textbf{0.0883}$ & $0.0862$ & $\textbf{0.2301}$ & $0.2268$ & $0.2274$ \\
          & $(<0.01s)$ & $(0.06s)$ & $(0.02s)$ & $(<0.01s)$ & $(0.07s)$ & $(0.02s)$ \\\hline
    \multirow{2}[0]{*}{User Knowledge} & $\textbf{0.3403}$ & $0.3398$ & $0.3398$ & $\textbf{0.4108}$ & $0.4086$ & $0.4083$ \\
          & $(<0.01s)$ & $(0.05s)$ & $(0.02s)$ & $(<0.01s)$ & $(0.05s)$ & $(0.02s)$ \\\hline
    \multirow{2}[0]{*}{Vinnie} & $0.4568$ & $0.4594$ & $\textbf{0.4597}$ & $0.3665$ & $0.3666$ & $\textbf{0.3677}$ \\
          & $(<0.001s)$ & $(<0.01s)$ & $(<0.01s)$ & $(<0.001s)$ & $(<0.01s)$ & $(<0.01s)$ \\\hline
    \multirow{2}[0]{*}{PRNN} & $\textbf{0.2093}$ & $0.2093$ & $0.2093$ & $\textbf{0.2913}$ & $0.2912$ & $0.2912$ \\
          & $(<0.001s)$ & $(0.01s)$ & $(<0.01s)$ & $(<0.001s)$ & $(<0.01s)$ & $(<0.01s)$ \\\hline
    \multirow{2}[0]{*}{Iris} & $\textbf{0.6246}$ & $0.5681$ & $0.5826$ & $\textbf{0.8014}$ & $0.7316$ & $0.7427$ \\
          & $(<0.001s)$ & $(<0.01s)$ & $(<0.01s)$ & $(<0.001s)$ & $(<0.01s)$ & $(<0.01s)$ \\\hline
    \multirow{2}[0]{*}{Transplant} & $0.7524$ & $\textbf{0.7687}$ & $0.7543$ & $\textbf{0.7248}$ & $0.7175$ & $0.7217$ \\
          & $(<0.001s)$ & $(<0.01s)$ & $(<0.01s)$ & $(<0.001s)$ & $(<0.01s)$ & $(<0.01s)$ \\\hline
    \end{tabular}}%
\caption{\textit{Comparison of GS, MS++, and $\alpha$-MS++ based on ARI and AMI scores over $14$ datasets}. These scores are calculated after tuning the bandwidth on each dataset for each algorithm separately. Best scores are reported in bold fonts. Compared to other algorithms, GS provides the best scores in ARI and AMI scores over most of the dataset with faster runtime speed. GS is 40x and 20x faster than MS++ and $\alpha$-MS++, respectively. Note that the parameter $h$ is tuned based on silhouette score for all algorithms.}
\label{tab:2}%

\end{table}%
\section{Theoretical Analysis}
In this section, we discuss the theoretical properties of GS. Due to the page limit, we provide proofs of the corresponding theorems in the supplementary file.

Before theoretical analysis, it is worth noting that GS also uses grid-based partitioning of the data space similar to MS++. Therefore, from Theorem 1 of \cite{jang2021meanshift++}, we can say that the GS approach can statistically perform at least similar to the other density function estimators. 
\subsection{Convergence Guarantee}
For any dataset $X(=\{x_1,x_2,\ldots,x_n\}) \in \mathbb{R}^d$, let us define a mapping $g^{(t)}: X \gets \mathcal{C}^{(t)}$ at $t$-iteration such that each data point $x_i \in X$ is assigned to one of the $k^{(t)}$ active grid cells (clusters) $c_i^{(t)} \in \mathcal{C}^{(t)}$. Therefore,
\begin{equation}
    \begin{split}
        & \mathcal{C}^{(t)} = \{c_1^{(t)},c_2^{(t)},\ldots,c_{k^{(t)}}^{(t)}\},~\mbox{and}\\
        & c_i^{(t)} \cap c_j^{(t)} = \phi, \forall i,j \in \{1,2,\ldots,k^{(t)}\},~ i \neq j.
    \end{split}
\end{equation}
Here, each active cell $c_i^{(t)}$ has a set of $1$-neighboring active grid cells, $\mathcal{Q}_{c_i}^{(t)} \subseteq \mathcal{C}^{(t)}$.
\begin{theorem}
For any given dataset $X \in \mathbb{R}^d$, the $\{\mathcal{C}^{(t)}\}_{t=1,2,\dots}$ estimated by successive proposed grid cells shifts attains convergence, i.e.
$\mathcal{C}^{(i)} == \mathcal{C}^{(i++)}$, where $i$ is a finite number.
\end{theorem}
\begin{theorem}
For any $X$, there exists $T \in \mathbb{N}$ such that $\mathcal{Q}_{c_i}^{(t)} = c_i^{(t)}$, $\forall i \in \{1,2,\ldots,k^{(t)}\}$ for all $t \geq T$.
\end{theorem}
The above two theorems confirm that GS attains convergence after a finite number of iterations when the active grid cells do not have any other active members in their 1-neighborhood to update their attributes any further.
\subsection{Convergence Rate}
In this subsection, we analyze the behavior of GS on mode seeking of a dataset sampled from a Gaussian distribution. We will prove that the number of active grid cells will form a non-increasing sequence, and centroids of these active grid cells will shrink towards the mean of the distribution with at least a cubic convergence rate.

Let $\phi(x; \mu, \Sigma)$ denotes a Gaussian probability density function, where $\mu$ and $\Sigma$ are the mean and dispersion matrix of the density function, respectively. To remove the dependency on the random process, we consider infinite samples generated from density $q(x) = \phi(x;0,diag(s_1^2,s_2^2,\ldots,s_d^2))$. 
\begin{theorem}
For dataset $X = \{x_1,x_2,\ldots,x_n\}$ where $x_i \sim \mathcal{N}(0,diag(s_1^2,\ldots,s_d^2))$, let centroids $\{c_i^{(t+1)}\}_{i=1}^{k^{(t+1)}} \sim \int y p^{(t)}(y|c)dy$, where $p^{(t)}()$ represents the distribution of $\{c_i^{(t+1)}\}_{i=1}^{k^{(t+1)}}$ and $p^{(t)}(y|c) = k(z-y)q^{(t)}(y)/p^{(t)}(z)$. Then (i) $\{c_i^{(t+1)}\}_{i=1}^{k^{(t+1)}} \sim \mathcal{N}\left(0,diag\left( \left( s_1^{(t+1)}\right)^2,\ldots,\left(s_d^{(t+1)}\right)^2 \right)\right)$, with $s_j^{(t+1)} = \left( 1+ 2.25\frac{h^2}{s^2}\right)^{-1}s_j^{(t)}$ and (ii) $k^{(t+1)} = \prod_{j = 1}^d \left( \floor*{\frac{6s_j^{(t+1)}}{h}}+1\right)$, where $\{k^{(t)}\}_{t=1}^{\infty}$ is non-decreasing sequence that converges to $1$.
\end{theorem}
\section{Results and Analysis}
To verify the effectiveness of GS, we carry out comparative experiments over three different scenarios: clustering, image segmentation, and object tracking. Source code of GS can be downloaded from \href{https://github.com/abhisheka456/GridShift}{https://github.com/abhisheka456/GridShift}.
\subsection{Clustering Application}
In this experiment, we investigate the performance of GS against MS++ and $\alpha$-MS++ on various clustering datasets. For a fair comparison, we considered the same datasets utilized in the original MS++'s paper~\cite{jang2021meanshift++}. We implement GS in Cython and the comparative study is performed using the Cython implementation of MS++ provided in \cite{jang2021meanshift++} and $\alpha$-MS++ done by us. Here, we skip the comparison between GS and MS as MS++ and $\alpha$-MS++ are the already available faster variants of MS~\cite{jang2021meanshift++, park}.

We independently tune parameter $h$ of GS over the range of $(0,1]$ for each dataset.  This is done by computing the Silhouette score for different values of the parameter $h$ within the specified range and choosing the one that provides the maximum Silhouette score.  To make a fair comparison, the MS++ and $\alpha$-MS++ parameter tuning processes are carried out in the same manner.  The obtained value of $h$ for all algorithms is reported in Table \ref{tab:new1}.
\begin{table}[!h]
    \centering
    \begin{tabular}{clccc} 
    \hline
    S.N & Dataset             & GS   & MS++ & $\alpha$MS++  \\ 
    \hline\hline
    1.  & Phone Gyroscope     & 0.68 & 0.65 & 0.67          \\
    2.  & Phone Accelerometer & 0.55 & 0.53 & 0.58          \\
    3.  & Watch Accelerometer & 0.72 & 0.72 & 0.72          \\
    4.  & Watch Gyroscope     & 0.45 & 0.42 & 0.46          \\
    5.  & Still               & 0.25 & 0.25 & 0.25          \\
    6.  & Skin                & 0.84 & 0.84 & 0.83          \\
    7.  & Wall Robot          & 0.38 & 0.38 & 0.38          \\
    8.  & Sleep Data          & 0.55 & 0.51 & 0.53          \\
    9.  & Balance Scale       & 0.62 & 0.61 & 0.62          \\
    10. & User Knoweldge      & 0.78 & 0.78 & 0.78          \\
    11. & Vinnie              & 0.80 & 0.80 & 0.80          \\
    12. & PRNN                & 0.43 & 0.43 & 0.43          \\
    13. & Iris                & 0.78 & 0.77 & 0.77          \\
    14. & Transplant          & 0.48 & 0.49 & 0.48          \\
    \hline
    \end{tabular}
    \caption{Tuned value of $h$ of GS, MS++, and $\alpha$-MS++ for each dataset used in experiment.}
    \label{tab:new1}
\end{table}

We employ the following two indices: Adjusted Mutual Information (AMI)~\cite{vinh2010information} and Adjusted Rand Index (ARI)~\cite{hubert1985comparing}, to determine the quality of the clustering results. These two indices are the popular way of analyzing clustering performance by comparing the calculated labels to actual labels of the clusters~\cite{jang2019dbscan++}. 

As discussed earlier, the computational complexity of GS is linear to the number of active grid cells and exponential with respect to the number of features. Therefore, we implement our algorithm with MS++ and $\alpha$-MS++ on $14$ standard low-dimensional datasets where the number of data points ranges from millions to  $100$~\cite{jang2021meanshift++}. Note that five datasets out of the 19 used in \cite{jang2021meanshift++} were excluded from our experiment because they were of very low-scale (less than 150) in size. The outcomes in terms of AMI, ARI, and runtime of all algorithms are depicted in Table \ref{tab:2}. As shown in Table \ref{tab:2}, GS provides better results with 40x and 20x faster runtime than MS++, and $\alpha$-MS++, respectively, on most datasets. Therefore, GS outperforms MS++ and $\alpha$-MS++ on these datasets in terms of clustering quality and runtime efficiency.


We also analyze the effect of the bandwidth setting on the performance of GS, MS++, and $\alpha$-MS++ in terms of clustering accuracy and runtime. We depict the outcome of this analysis in Figure \ref{fig:3_new}. As shown in Figure \ref{fig:3_new}, GS provides better clustering results than other algorithms over different bandwidth settings with faster runtime. This analysis indicates that GS returns more accurate clustering results than MS++ and $\alpha$-MS++. The main reason behind the better performance of GS is that it updates centroids of active grid cells sequentially, which generate gradient-ascent shifting steps more stably in the next iteration.
\begin{figure}
    \centering
    \includegraphics[width=0.9\linewidth]{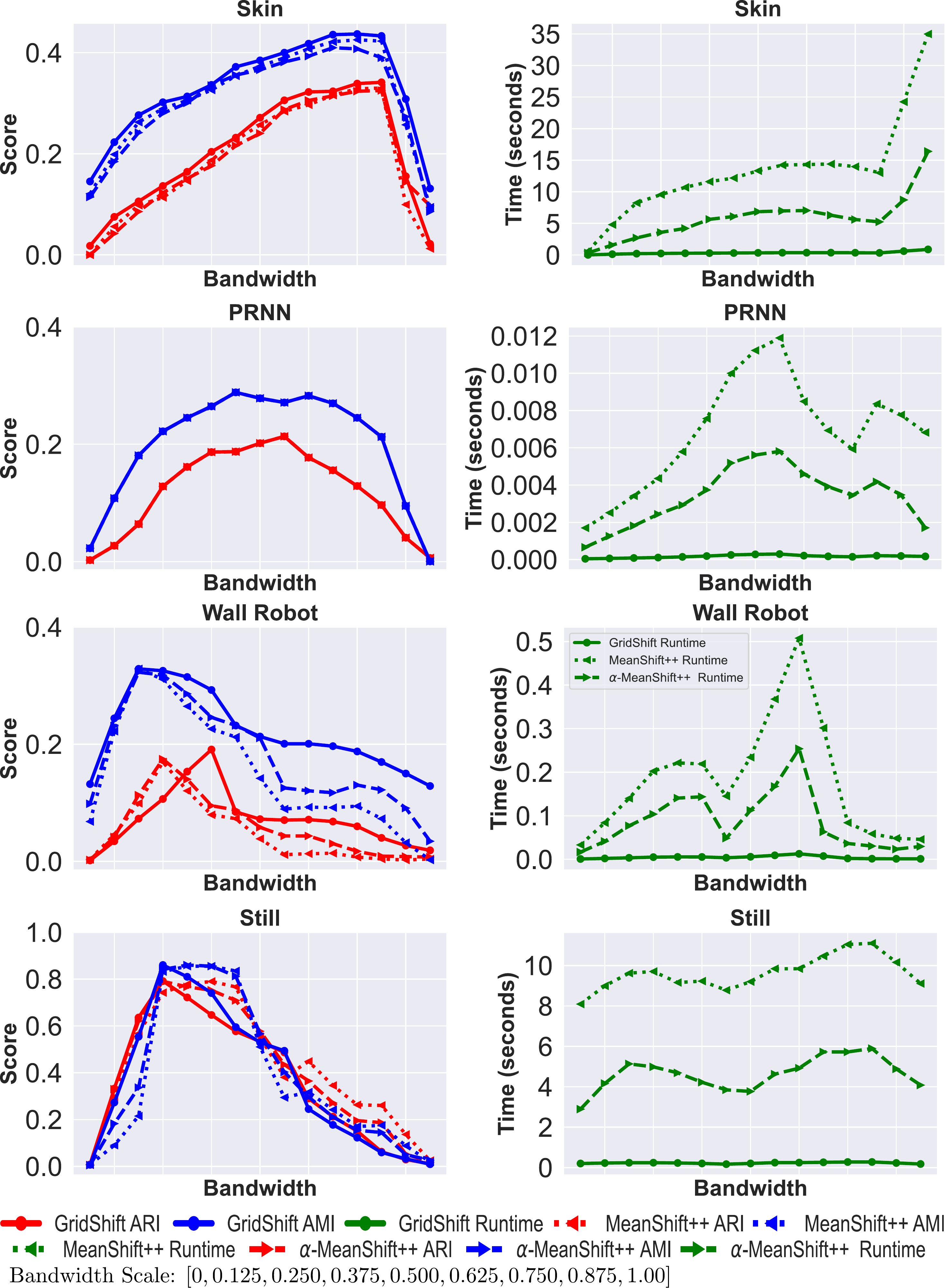}
    \caption{\textit{Comparison of GS, MS++, and $\alpha$-MS++ on four real-world datasets over a wide bandwidth range}. This figure shows the behavior of all algorithms over the different settings of bandwidth. We can see here that GS performs better than MS++ and $\alpha$-MS++ despite being faster over different bandwidth values.}
    \label{fig:3_new}
    
\end{figure}
\begin{figure}[!ht]
    \centering
    \includegraphics[width = 0.8\linewidth]{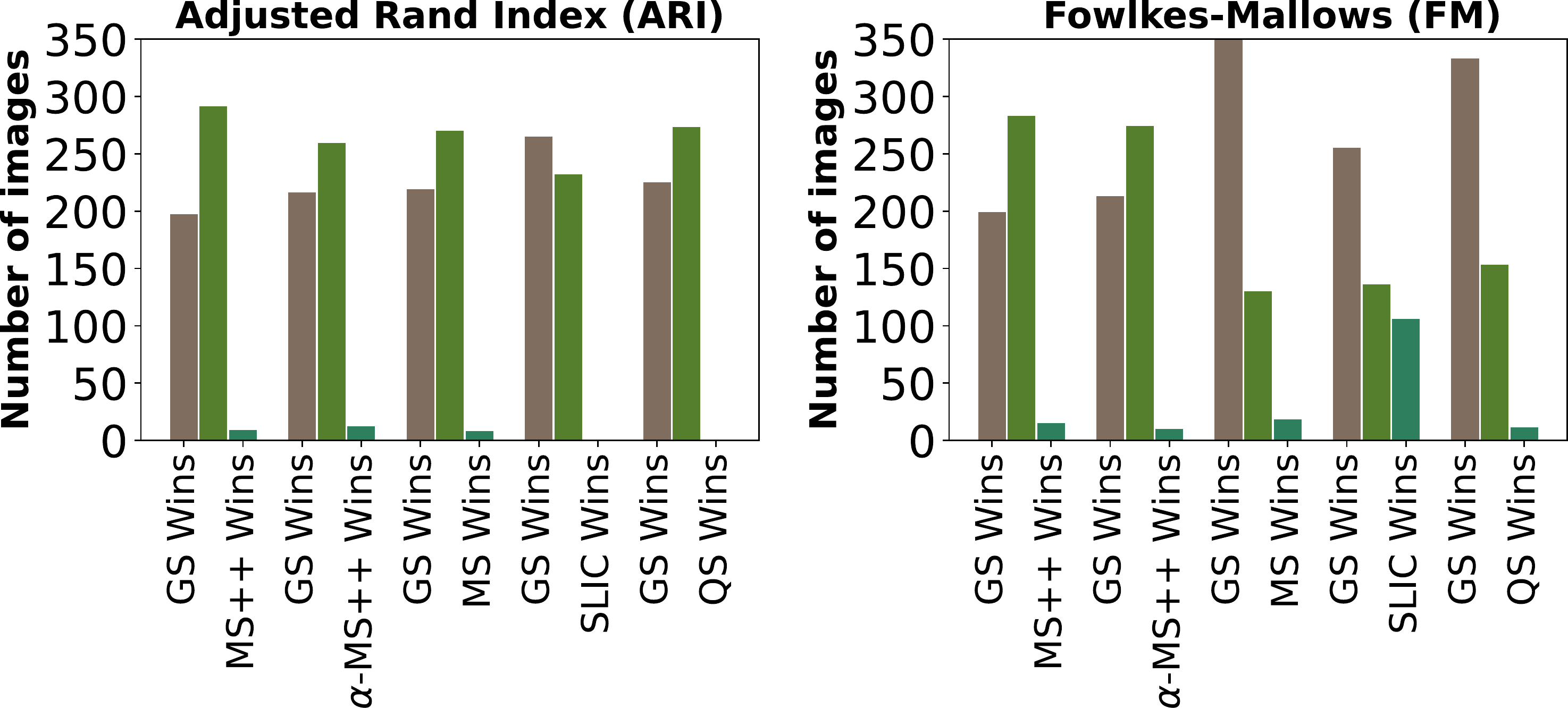}
    \caption{\textit{Comparison of all algorithms on image segmentation of BSDS500 benchmark images using the two popular metrics ARI and FM.} For each contender, we plot bars of GS winning, contender winning, and both are similar within 1\% of their scores for AMI and FM separately. Here, GS performs better than other algorithms in most cases.}
    \label{fig:2}
    
\end{figure}

\subsection{Image Segmentation}
In this section, we analyze the performance of GS compared to algorithms on unsupervised image segmentation. We consider the following algorithms as state-of-the-art:  Quickshift~\cite{Vedaldi2008QuickSA}, SLIC ($k$-means)~\cite{achanta2012slic}, and MS++~\cite{jang2021meanshift++}. For Quickshift, and SLIC, we use Python SciKit-Image Library~\cite{van2014scikit}, and Cython implementation~\cite{jang2021meanshift++} of MS++ is utilized here.  In this experiment, we skip the MS algorithm as the performance of MS++ is similar to MS, but with $10000$x speedup.

 We experiment on an image segmentation benchmark dataset BSDS500 for quantitative performance analysis of GS compared to other state-of-the-art algorithms. This benchmark dataset was also used in MS++ paper~\cite{jang2021meanshift++} and it has 500 images with six different human-labeled segmentation.  To make an objective comparison, we have tuned all parameters of image segmentation (unsupervised) algorithms, including GS's $h$. For the image segmentation experiment, we convert all images into 3-dimensional NumPy arrays of (R, G, B) values of pixels.  We implement GS along with MS++, $\alpha$-MS++, MS, SLIC, and Quickshift on each image and calculate ARI and Fowlkes-Mallows (FM)~\cite{fowlkes1983method} clustering scores for each human-labeled segmentation separately. For comparison, we take average scores of ARI and FM based on all six human-labeled segmentation. Note that we performed the GS image segmentation with both $(r,g,b)$ and $(r,g,b,x,y)$ inputs. However, the results were not significantly different. Therefore, we only reported results with $(r,g,b)$ cases. GS with $(r,g,b,x,y)$ performed similarly to GS with $(r,g,b)$ on $488$ and $492$ out of $500$ BSDS500 images in terms of ARI and FW, respectively. In Figure \ref{fig:2}, we show the performance comparison of GS with other algorithms in terms of these scores, and we also report the average runtime of all algorithms in Table \ref{tab:new}. As shown in Figure \ref{fig:2} and Table \ref{tab:new}, GS provides better results than other algorithms in most cases. Moreover, it is faster than other algorithms: $80$x, $40$x, $50000$x, $2.5$x, and  $900$x faster than MS++, $\alpha$-MS++, MS, SLIC, and Quickshift, respectively.

\begin{table}[!ht]
\scriptsize
\centering
    \begin{tabular}{|c|c|}\hline
     Algorithm    & Average Runtime (seconds) \\\hline\hline
      GridShift   & 0.031\\
      MS++ & 2.561\\
      $\alpha$-MS++ & 1.281\\
      MS & 1690.213\\
      SLIC ($k$-meanss) & 0.078\\
      Quickshift & 28.031\\\hline
    \end{tabular}
\caption{\textit{Comparison of the runtime of all algorithms on image segmentation of BSDS500 benchmark images.} The average runtime of GS is better than other contenders by 80x, 40x, 50000x, 2.5x, and 900x than MS++, $\alpha$-MS++, MS, SLIC, and Quickshift, respectively.}
\label{tab:new}

\end{table}
\subsection{Object Tracking}
This section demonstrates the robustness of the GS-based object-tracking algorithm compared to CS and MS++. In this experiment, we consider three visual tracker benchmark problems from \cite{WuLimYang13}.  In object tracking, Note that we have tuned the parameters of GS, MS++, and CS separately for each video sequence dataset, not for each frame of the sequence individually.
\begin{figure*}[!ht]
    \centering
    \includegraphics[width=0.7\linewidth]{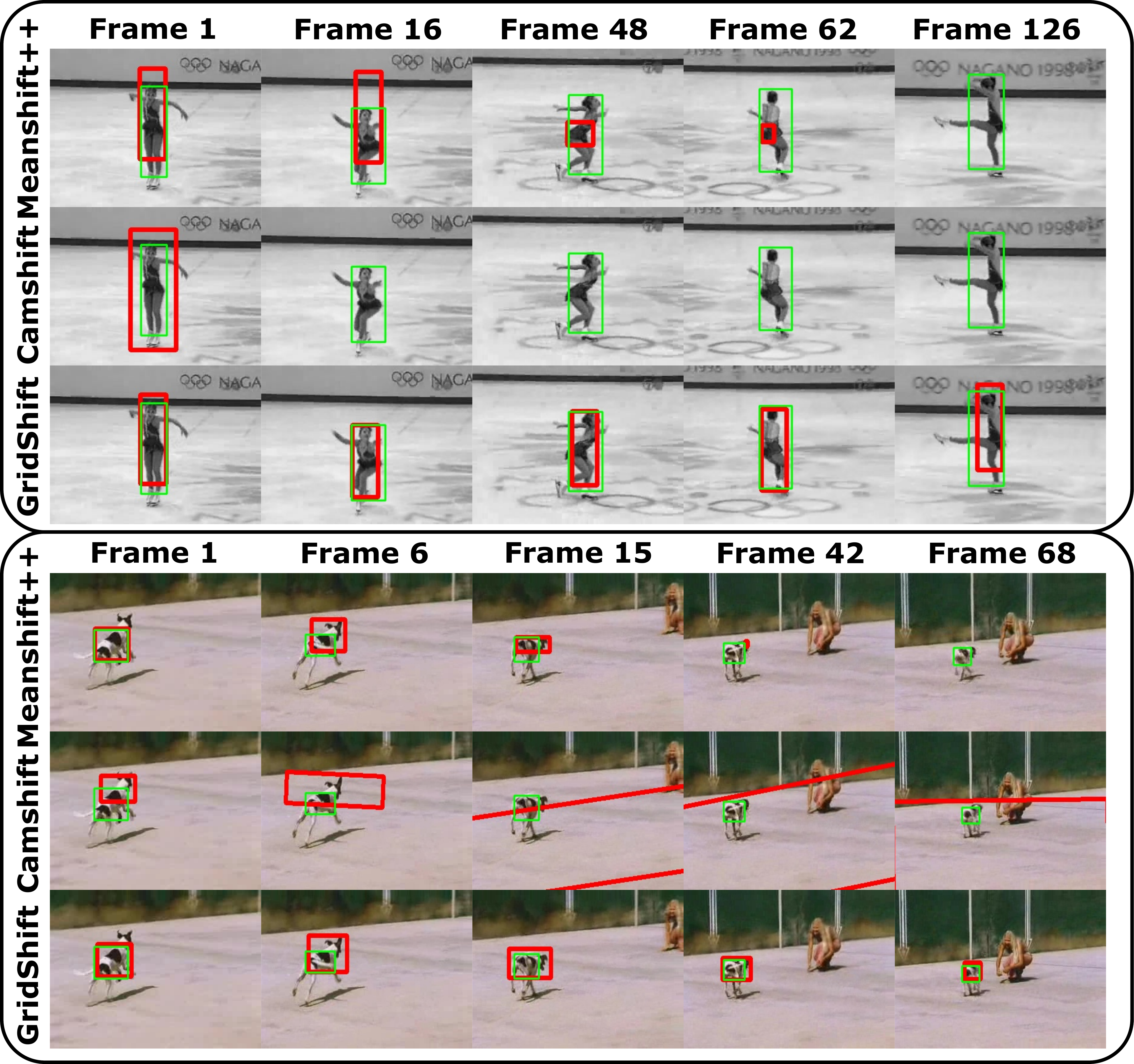}
    \caption{\textit{Comparison of meanshift++, camshift, and GridShift on object tracking.} Here, camshift fails to track an object in both cases because the user has to provide color ranges for the object to the algorithm for creating a color histogram. However, this information is generally incomplete, biased, and inaccurate. On the other hand, MS++ and GS use grid cells of the color pixels of objects found after image segmentation. Although MS++ performs better than camshift, it cannot track fast-moving objects after some frames as it cannot adapt the tracking window effectively. GS performs better than both algorithms as it gains more stability after incorporating the proposed tracking window adaptation scheme.}
    \label{fig:3}

\end{figure*}
In Figure \ref{fig:3}, we depict the performance of these algorithms on all visual tracker benchmark problems. As shown in Figure \ref{fig:3}, CS tends to quickly fail to track the object due to surroundings, especially when the surroundings have objects of similar color. The main reason behind that CS tracking depends upon the color histogram calculated initially for the given object.  Since CS cannot fastly calculate the new histogram for the more recent frame due to scene change, it failed consistently in all three object tracking problems. On the other hand, MS++ and GS do not mislead by the surroundings as they can adapt changes in color distributions by adding neighboring similar grid cells in the reference set by finding them in linear time. This property makes MS++ and GS more robust for changes in lighting and colors. However, MS++ cannot adapt the tracking window when the object is moving fastly. As a result, it loses the object, or the tracking window reduces to a point size. This issue makes MS++ unsuitable for the real-life application where the object is moving fastly. We resolve this issue in GS by incorporating a tracking window adaptation scheme. Thus, GS can track the fast-moving objects accurately by updating the area of the tracking window when it misses the object, or the object changes its size during a scene change.
By analyzing the performance of all algorithms, we can conclude that GS is a better visual tracker than MS++ and CS.

 We also undertook the experiment on the full OTB100 dataset~\cite{WuLimYang13}. We show a detailed comparison of GS with the recently developed tracking algorithm on OTB100s in Table \ref{revtab:1}.

\begin{table}[!ht]
    \centering
    \caption{Comparative performance on OTB100.}
    \resizebox{0.7\linewidth}{!}{\begin{tabular}{ccc|ccc}\hline
       Tracker  & AUC & P &  Tracker & AUC & P  \\\hline
       \multicolumn{6}{c}{ Supervised Learning-based Trackers} \\\hline
       GCT (CVPR, 2019)  & 0.647 & 0.853 &  SiamRPN (CVPR, 2018) & 0.637 & 0.851 \\
       GradNet (ICCV, 2019) & 0.639 & 0.861  & ACT (ECCV, 2018) & 0.625 & 0.859  \\\hline
       \multicolumn{6}{c}{ Unsupervised Learning-based Trackers} \\\hline
       USOT (ICCV, 2021) & 0.589 & 0.806 & USOT* (ICCV, 2021) & 0.574 & 0.775 \\
       LUDT (IJCV, 2021) & 0.602 & 0.769 & LUDT+ (IJCV, 2021) & 0.639 & 0.843 \\
       AlexPUL (CVPR, 2021) & 0.551 & -- & USOT-GS (Ours) & \textbf{0.651} & \textbf{0.872}\\\hline
       \multicolumn{6}{c}{ Unsupervised Iterative-based Trackers} \\\hline
       DSST (2014) & 0.518 & 0.689 & KCF (2015) & 0.485 & 0.696 \\
       MS++ (CVPR, 2021) & 0.326 & 0.538 & GS (Ours) & 0.638 & 0.866\\\hline
    \end{tabular}}
    \label{revtab:1}
\end{table}
Here, we list results on OTB100 in terms of area under the curve (AUC) and average distance precision (P) at 20 pixels. In Table \ref{revtab:1}, we highlight performance of GS against 12 state-of-the-art tracking models. As per the results GS remains consistently superior to its peers. Moreover, we will also show how GS can provide user-independent labels for unlabeled training datasets. We have incorporated GS as candidate box generation in USOT tracker (ICCV, 2021), named USOT-GS. As per Table \ref{revtab:1}, USOT-GS performs better than USOT and USOT* on OTB100 datasets. GS's tracking accuracy is competitive with the current state-of-the-art due to the following reasons. (1) To detect the target object in subsequent sequences, GS uses the grid cells bin, which is the result of clustering the target object. On the other hand, a color histogram requires a precomputed mask or color range that is entirely based on the input of the user, which can be flawed and even, biased. In this case, grid cells can serve as a better alternative to detect targets. (2) Secondly, because GS achieves speedup in clustering, it may adapt the color distribution or scene changes by finding and adding new grid cells to reference grid cell bins in linear time, thus, being more robust. (3) We also suggest a window size adaptation scheme to gradually update the window box as objects move closer or farther away from the window. 

\section{Conclusion}
We introduce a simple yet faster mode-seeking algorithm, GS, for low-dimensional large-scale datasets. We implement this algorithm in unsupervised clustering and image segmentation applications. GS functions better than MS++ in terms of accuracy and efficiency, while GS is 40x faster than MS++ (40000x faster than MS). Moreover, GS provides better results than other popular image-segmentation algorithms like SLIC, Quickshift, and Felzenszwalb regarding accuracy and computation time.

We also propose a new object-tracking algorithm based on GS. The performance of the proposed algorithm is better than the popular CamShift algorithm and MS++ as it can adapt color distribution gradually with scene change effectively by recalculating the grid cell for the newer frames. Hence, GS can apply for modern computer vision problems due to lower computational cost, and also it provides more accurate results.\\
\textbf{Acknowledgements:} This work was supported by the Basic Science Research Program through the National Research Foundation of Korea (NRF) funded by the Ministry of Education (2021R1I1A3049810).
{\small
\bibliographystyle{ieee_fullname}
\bibliography{egbib}
}
\begin{center}
\section*{Supplementary File}
\end{center}
\section{Theoretical Analysis of GridShift}
In essence, Mean Shift is a clustering algorithm based on the following intuition. For dataset $X = \{x_i\}_{i=1}^n \subseteq \mathbb{R}^n$, let's assume the probability density function (PDF) is $p(z)$. In this dataset, we can expect $k$ clusters if PDF $p(z)$ has $k$ modes. Additionally, suppose an optimization algorithm, such as gradient descent, is run with a starting point $x_m$ and converges to the $j$-th mode. In that case, it can be considered that the data point $x_m$ is part of the cluster that belongs to $j$-th mode~\cite{huang2018convergence}.

The PDF of datasets is not available in practice, and only the data points are accessible. To implement the intuition mentioned above, one must first estimate the PDF, a process known as density estimation~\cite{scott2015multivariate}. Kernel density estimators (KDEs) are the most popular method for estimating density. Let $K(z)$ be the Kernal function satisfying
\begin{equation}\label{equ1}
    K(z) \geq 0,~\mbox{and}~\int K(z)dz = 1,
\end{equation}
its KDE is 
\begin{equation}
    \hat{p}(z) = \frac{1}{n} \sum_{i=1}^n K(z-x_i).
\end{equation}
In mean shift (MS) algorithms, there are two main kernels that are employed: the Gaussian kernel 
\begin{equation*}
    K_G(z;h) = c. \exp\left(-\frac{||z||^2}{2h^2}\right),
\end{equation*}
and the Epanechnikov kernel
\begin{equation*}
    K_E(z;h) = c. \max\left\{0, 1- \frac{||z||^2}{h^2}\right\},
\end{equation*}
where $c$ is the constant to ensure that the kernel will integrate to $1$~\cite{fukunaga1975estimation,cheng1995mean, comaniciu2002mean}.

By iterating through the following equation, initialized at each $x_i$, the MS algorithm attempts to estimate modes of $\hat{p}(z)$ based on the its KDE~~\cite{comaniciu2002mean}:
\begin{equation}
    z \gets \frac{1}{\sum_{i=1}^{n} g\left(||z-x_i||^2 \right)} \sum_{i = 1}^n g\left(||z-x_i||^2 \right)x_i
\end{equation}
where $g(||z||^2) \propto -K'(||z||^2)$, i.e. kernel $K(||z||^2)$ is the shadow of the kernel $g(||z||^2)$.

Even though MS exists, GridShift (GS) uses grid-based KDE. The definition of grid-based KDE is as follows.
\begin{equation}
    \hat{p}(z) = \frac{1}{n}\sum_{i=1}^n K_{GS}(z-x_i),
\end{equation}
where
\begin{equation}
    K_{GS}(z-x_i;h) = \begin{cases} c.\left(a - ||z-x_i||^2\right),&\mbox{if}~\mathcal{M}(z,x_i;h) \leq 1\\
    c.a,& \mbox{otherwise},
    \end{cases}
\end{equation}
\begin{equation}
  \mathcal{M}(z,x_i;h) = \mbox{max}\left\{\left\vert\floor*{\frac{x_{i,1}}{h}}-\floor*{\frac{z_{1}}{h}}\right\vert,\ldots,\left\vert\floor*{\frac{x_{i,d}}{h}}-\floor*{\frac{z_{d}}{h}}\right\vert\right\}.  
\end{equation}
Here $a$ and $c$ are positive constants to satisfy kernel function conditions defined in Eqn \ref{equ1}. Therefore, $g(.)$ can be defined as follows:
\begin{equation}
    g(||z-x_i||^2;h) = \begin{cases} 1, &\mbox{if}~\mathcal{M}(z,x_i;h) \leq 1\\ 
    0,& \mbox{otherwise}
    \end{cases}
\end{equation}
\subsection{Function Analysis}
GS attempts to find local maxima (modes) of the KDE $\hat{p} = \sum K_{GS}(z-x_i;h)$. We omit constants and scaling to define functions $\phi$ and $f$ instead of $K_{GS}$ and $\hat{p}$ since they do not affect optimization:
\begin{equation}\label{eqn:8}
    \begin{split}
        & \phi(z-x_i) = \begin{cases} ||z-x_i||^2, & \mbox{~if} ~\mathcal{M}(z,x_i;h) \leq 1\\ a, & \mbox{~otherwise}\end{cases}\\
        & f(z) = \sum_{i=1}^n \phi(z-x_i).
    \end{split}
\end{equation}
A mode of $\hat{p}$ corresponds to the local minima of $f(z)$. Let's examine the properties of the loss function $f(z)$.
\begin{lemma}
Let us define $\mathcal{P}(z) = \{i: \mathcal{M}(z,x_i;h) < 1\}$, then we get 
\begin{equation}
    \begin{split}
        &\nabla f(z) = \sum_{i \in \mathcal{P}(z)} 2(z-x_i)\\
        &\nabla^2 f(z) = 2 \vert \mathcal{P}(z) \vert I.
    \end{split}
\end{equation}
If $\mathcal{P}(z)$ is not an empty set, then $f(z)$ is strongly convex; otherwise, it is locally convex.
\end{lemma}
\begin{proof}
Here, the function $f(z)$ is local convex because of $\nabla^2 f(z) \geq 0$. Further, if $|\mathcal{P}(z)| \neq \emptyset$, then $\nabla^2 f(z) \geq I$, which means that $f(z)$ is strongly convex locally.
\end{proof}
\begin{lemma}
If a point $z^*$ is a local minimum for $f(z)$, then $\mathcal{P}(z^*) \neq \emptyset$ and $z^* = \frac{1}{\vert\mathcal{P}(z^*)\vert} \sum_{i \in \mathcal{P}(z^*)}x_i $.
\end{lemma}
\begin{proof}
In order to be stationary, a point $z^*$ must meet the following criteria: 
\begin{equation}
    \begin{split}
        & \nabla f(z^*) = 0\\
        \Rightarrow & z^* = \frac{1}{\vert \mathcal{P}(z^*) \vert} \sum_{i \in \mathcal{P}(z^*)} x_i
    \end{split}
\end{equation}
If $\mathcal{P}(z^*) \neq \emptyset$, then $\nabla^2f(z^*) > 0$; therefore, $z^*$ is a local minimum. In the case of $\mathcal{P}(z^*) \neq \emptyset$, $f(z^*) = n.a$ (a global maximum).
\end{proof}
\begin{definition}
If two points $y$ and $z$ lie in the same grid cell, i.e. $\floor*{\frac{y}{h}} = \floor*{\frac{z}{h}}$, then
\begin{equation}
    y^* = z^* = \frac{1}{\vert \mathcal{P}(z^*)\vert} \sum_{i \in \mathcal{P}(z^*)}x_i.
\end{equation}
Therefore, all the points within a grid cell have the same local minima.
\end{definition}
Due to the same $\mathcal{P}(.)$ value for all grid cell points,  these points have the same local minima. This property of KDE $K_{GS}$ motivates us to develop a new framework, GridShift (GS), faster than the original MS. In GS, we called the local minima of a grid cell the centroid. Within each iteration, centroids (local minima) are updated. Utilizing the centroids of the previous iteration, we update these centroids. 
\subsection{Convergence Guarantee}
Let us define a mapping $g^{(t)}: X \gets \mathcal{C}^{(t)}$ for any dataset $X(=\{x_1,x_2,\ldots,x_n\}) \in \mathbb{R}^d$, such that each data point $x_i \in X$ is assigned to one of the $k^{(t)}$ active grid cells (clusters) $c_i^{(t)} \in \mathcal{C}^{(t)}$. Therefore,
\begin{equation}
    \begin{split}
        & \mathcal{C}^{(t)} = \{c_1^{(t)},c_2^{(t)},\ldots,c_{k^{(t)}}^{(t)}\},~\mbox{and}\\
        & c_i^{(t)} \cap c_j^{(t)} = \phi, \forall i,j \in \{1,2,\ldots,k^{(t)}\},~ i \neq j.
    \end{split}
\end{equation}
Here, each active cell $c_i^{(t)}$ has a set of $1$-neighboring active grid cells, $\mathcal{P}_{c_i}^{(t)} \subseteq \mathcal{C}^{(t)}$.

\begin{corollary}
The value of $f(r_i)$ obtained by GS is strictly decreasing unless $r_i^{(t)} = r_i^{(t-1)}$.
\end{corollary}
\begin{proof}
In GS, we update
\begin{equation}
    r_j^{(t)} = \frac{\sum_{i \in \mathcal{P}(r_j^{(t-1)})}m_i^{(t-1)}r_i^{(t-1)}}{\sum_{i \in \mathcal{P}(r_j^{(t-1)})}m_i^{(t-1)}},
\end{equation}
where $m_i$ represents number of data points resident in $i$th grid cell. At a particular point $\tilde{z}$, define $\overline{f}(z\vert\tilde{z})$ using following equation.
\begin{equation}
        \overline{f}(z\vert\tilde{z}) = \sum_{i \in \mathcal{P}(\tilde{z})} m_i \Vert z-r_i \Vert^2 +\left( n - \sum_{i \in \mathcal{P}(\tilde{z})}m_i\right)a.
\end{equation}
Then,
\begin{equation}\label{eqn:15}
    \begin{split}
        & f(r_j^{(t-1)}) - f(r_j^{(t)})\\
        & \geq \overline{f}(r_j^{(t-1)}\vert r_j^{(t-1)}) - \overline{f}(r_j^{(t)} \vert r_j^{(t-1)})\\
        & = \sum_{i \in \mathcal{P}(r_j^{(t-1)})} m_i \Vert r_j^{(t-1)}-r_i^{(t-1)} \Vert^2 - \sum_{i \in \mathcal{P}(r_j^{(t-1)})} m_i \Vert r_j^{(t)}-r_i^{(t-1)} \Vert^2\\
        & = \sum_{i \in \mathcal{P}(r_j^{(t-1)})} m_i \left(\Vert r_j^{(t-1)}-r_i^{(t-1)} \Vert^2- \Vert r_j^{(t)}-r_i^{(t-1)} \Vert^2\right)\\
        & = \left( \sum_{i \in \mathcal{P} \left( r_j^{(t-1)}\right)} m_i\right) \Vert r_j^{(t-1)}-r_j^{(t)}\Vert^2 > 0
    \end{split}
\end{equation}
Therefore, $f(r_j^{(t-1)}) > f(r_j^{(t)})$, unless $f(r_j^{(t-1)}) = f(r_j^{(t)})$ for $r_j^{(t)} = r_j^{(t-1)}$.
\end{proof}
\begin{theorem}
For any given dataset $X \in \mathbb{R}^d$, the $\{\mathcal{C}^{(t)}\}_{t=1,2,\dots}$ estimated by successive proposed grid cells shifts attains convergence, i.e.
$\mathcal{C}^{(i)} == \mathcal{C}^{(i++)}$, where $i$ is a finite number.
\end{theorem}
\begin{proof}
From corollary 1, since the value of $f(r)$ is monotonically non-increasing, GS attains convergence to the local minima of function defined in Eqn. (\ref{eqn:8}). As we know that Set $\mathcal{C}$ contains the centroid of active grid cells (local minima). Therefore, sequence $\{\mathcal{C}^{(t)}\}_{t=1,2,\dots}$ attains convergence.

From Eqn. (\ref{eqn:15}), we can have
\begin{equation}
    f(r_j^{(t-1)}) -f(r_j^{(t)}) \geq \Vert r_j^{(t-1)} -r_j^{(t)} \Vert^2,
\end{equation}
After summing both sides for $t = 1,\ldots,i$, we get
\begin{equation}
    f(r_j^{(0)}) - f(r_j^{(i)}) \geq \sum_{t=1}^i \Vert r_j^{(t-1)} - r_j^{(t)}\Vert^2.
\end{equation}

As we know, the right-hand side of the above equation is positive unless convergence is attained. if we calculate the maximum value of $\lambda > 0$ such that
\begin{equation}
    \Vert r_j^{(t-1)} -r_j^{(t)}\Vert \geq \lambda > 0,~\forall t=1,\ldots,i,
\end{equation}
then
\begin{equation}
    i \leq \frac{f(r_j^{(0)})-f(r_j^{(i)})}{\lambda}
\end{equation}
which is a finite number.
\end{proof}
\begin{theorem}
For any $X$, there exists $T \in \mathbb{N}$ such that $\mathcal{Q}_{c_i}^{(t)} = c_i^{(t)}$, $\forall i \in \{1,2,\ldots,k^{(t)}\}$ for all $t \geq T$.
\end{theorem}
\begin{proof}
As we know, at convergence, we have
\begin{equation}
    r_j^{(t+1)} = \frac{\sum_{i \in \mathcal{P}(r_j^{(t)})}m_i^{(t)}r_i^{(t)}}{\sum_{i \in \mathcal{P}(r_j^{(t)})}m_i^{(t)}}  = r_j^{(t)},
\end{equation}
that implies $\mathcal{P}(r_j^{(t)}) = j$, i.e. $\mathcal{Q}_{c_i}^{(t)} = c_i^{(t)}$.
\end{proof}
The above two theorems confirm that GS attains convergence after a finite number of iterations when the active grid cells do not have any other active members in their 1-neighborhood to update their attributes further.
\subsection{Convergence Rate}
In this subsection, we analyze the behavior of GS on mode seeking of a dataset sampled from a Gaussian distribution. We will prove that the number of active grid cells will form a non-increasing sequence, and centroids of these active grid cells will shrink towards the mean of the distribution with at least a cubic convergence rate.

Let $\phi(x; \mu, \Sigma)$ denotes a Gaussian probability density function, where $\mu$ and $\Sigma$ are the mean and dispersion matrix of the density function, respectively. To remove the dependency on the random process, we consider infinite samples generated from density $q(x) = \phi(x;0,diag(s_1^2,s_2^2,\ldots,s_d^2))$. 
\begin{theorem}
For dataset $X = \{x_1,x_2,\ldots,x_n\}$ where $x_i \sim \mathcal{N}(0,diag(s_1^2,\ldots,s_d^2))$, let centroids $\{c_i^{(t+1)}\}_{i=1}^{k^{(t+1)}} \sim \int y p^{(t)}(y|c)dy$, where $p^{(t)}()$ represents the distribution of $\{c_i^{(t+1)}\}_{i=1}^{k^{(t+1)}}$ and $p^{(t)}(y|z) = k(z-y)q^{(t)}(y)/p^{(t)}(z)$. Then (i) $\{c_i^{(t+1)}\}_{i=1}^{k^{(t+1)}} \sim \mathcal{N}\left(0,diag\left( \left( s_1^{(t+1)}\right)^2,\ldots,\left(s_d^{(t+1)}\right)^2 \right)\right)$, with $s_j^{(t+1)} = \left( 1+ 2.25\frac{h^2}{s_j^2}\right)^{-1}s_j^{(t)}$ and (ii) $k^{(t+1)} = \prod_{j = 1}^d \left( \floor*{\frac{6s_j^{(t+1)}}{h}}+1\right)$, where $\{k^{(t)}\}_{t=1}^{\infty}$ is non-decreasing sequence that converges to $1$.
\end{theorem}
\begin{proof}
To estimate the distribution of $\{c_i^{(t+1)}\}_{i=1}^{k^{(t+1)}}$, we have
\begin{equation}
    p^{(1)}(y|c^{(0)}) \propto \exp\left\{ -\frac{1}{2} \frac{\left\Vert y - \frac{c^{(0)}/(2.25h^2)}{/(2.25h^2)+1/(s^{(0)})^2} \right\Vert^2}{\frac{1}{1/(2.25h^2)+1/(s^{(0)})^2}}\right\}.
\end{equation}
As we know,
\begin{equation}
    c^{(1)} = E(y|x^{(0)}) = \left( \frac{c_1^{(0)}(s_1^{(0)})^2}{(s_1^{(0)})^2 + 2.25h^2},\ldots, \frac{c_d^{(0)}(s_d^{(0)})^2}{(s_d^{(0)})^2 + 2.25h^2}\right).
\end{equation}
Therefore, $c^{(1)}$ is also a Gaussian distribution with mean zero and standard deviation $s^{(1)} = \frac{(s^{(0)})^3}{(s^{(0)})^2+2.25h^2}$, which implies
\begin{equation}
    s_j^{(t+1)} = \frac{(s_j^{(t)})^3}{(s_j^{(t)})^2+2.25h^2} = \left( 1+ 2.25\frac{h^2}{(s_j^{(t)})^2}\right)^{-1}s_j^{(t)}
\end{equation}
Thus, standard deviation is decreasing with increase of iteration and become zero at convergence. We can estimate the number of active grid cells according to standard deviation of this distribution as follows.
\begin{equation}
 k^{(t+1)} = \prod_{j = 1}^d \left( \floor*{\frac{6s_j^{(t+1)}}{h}}+1\right).  
\end{equation}
We see that $k^{(t)}$ is a non-decreasing sequence and converges to $1$ when $s_j^{(t+1)}$ becomes 0.
\end{proof}

\section{Dataset}

\begin{table}[!h]
    \centering
    \scriptsize
    \begin{tabular}{clccc}\hline
     S.N & Dataset & $n$ & $d$ & $k$\\\hline\hline
     1. & Phone Gyroscope  & $13932632$ & $3$ & $7$ \\
     2. & Phone Accelerometer & $13062475$ & $3$ & $7$ \\
     3. & Watch Accelerometer & $3540962$ & $3$ & $7$ \\
     4. & Watch Gyroscope & $3205431$ & $3$ & $7$ \\
     5. & Still & $949983$ & $3$ & $6$ \\
     6. & Skin & $245057$ & $3$ & $2$ \\
     7. & Wall Robot & $5456$ & $4$ & $4$ \\
     8. & Sleep Data & $1024$ & $2$ & $2$ \\
     9. & Balance Scale & $625$ & $4$ & $3$ \\
     10. & User Knoweldge & $403$ & $5$ & $5$ \\
     11. & Vinnie & $380$ & $2$ & $2$ \\
     12. & PRNN & $250$ & $2$ & $2$ \\
     13. & Iris & $150$ & $4$ & $3$ \\
     14. & Transplant & $131$ & $3$ & $2$ \\\hline
    \end{tabular}
    \caption{Brief summary of datasets used in experiment. $n$: number of data points, $d$: number of features, and $k$: number of clusters.}
    \label{tab:1}
\end{table}

\end{document}